\documentclass[pdf-a,balance,colorlinks,upint,subscriptcorrection,varvw,mathalfa=cal=boondoxo, spanish,french,vietnamese,russian,greek]{asmeconf}

\hypersetup{%
	pdfauthor={John H. Lienhard},									  %
	pdftitle={ASME Conference Paper LaTeX Template},                  %
	pdfkeywords={engineering design, design requirements, large language models, benchmark, Formula SAE},%
	pdfsubject = {Describes the asmeconf LaTeX template},			  %
	pdflicenseurl={https://ctan.org/pkg/asmeconf},%
}

\usepackage{float}
\usepackage{mathtools}
\usepackage{booktabs}
\usepackage{multirow}
\usepackage{graphicx}
\usepackage{colortbl}
\usepackage{amsthm}
\usepackage{geometry}
\usepackage{gensymb}
\usepackage{makecell} %
\usepackage{pifont} %
\newtheorem{lemma}{Lemma}

\usepackage{tabularx}

\definecolor{customblue}{RGB}{149,203,251}
\usepackage{mdframed}
\newmdenv[
  topline=false,
  bottomline=false,
  rightline=false,
  linewidth=2pt,
  linecolor=customblue,
  backgroundcolor=gray!20,
  leftmargin=10pt,
  rightmargin=10pt,
  innertopmargin=10pt,
  innerbottommargin=10pt,
  font=\small
]{customquote}

\newcommand{\ours}{CAD-Coder}
\newcommand{\dataset}{GenCAD-Code dataset}
\newcommand{\datasetnoe}{GenCAD-Code}

\begin{document}

\title{\ours{}: An Open-Source Vision-Language Model for Computer-Aided Design Code Generation} %

\SetAuthors{%
	Anna C.\ Doris\affil{1}\CorrespondingAuthor{adoris@mit.edu},
	Md Ferdous Alam\affil{1},
    Amin Heyrani Nobari\affil{1},
    Faez Ahmed\affil{1}
	}

\SetAffiliation{1}{Massachusetts Institute of Technology, Cambridge, MA }

\maketitle
\pagestyle{plain}

\begin{abstract}

Efficient creation of accurate and editable 3D CAD models is critical in engineering design, significantly impacting cost and time-to-market in product innovation. Current manual workflows remain highly time-consuming and demand extensive user expertise. While recent developments in AI-driven CAD generation show promise, existing models are limited by incomplete representations of CAD operations, inability to generalize to real-world images, and low output accuracy. This paper introduces \ours{}, an open-source Vision-Language Model (VLM) explicitly fine-tuned to generate editable CAD code (CadQuery Python) directly from visual input. Leveraging a novel dataset that we created—\datasetnoe{}, consisting of over 163k CAD-model image and code pairs—\ours{} outperforms state-of-the-art VLM baselines such as GPT-4.5 and Qwen2.5-VL-72B, achieving a 100\% valid syntax rate and the highest accuracy in 3D solid similarity. Notably, our VLM demonstrates some signs of generalizability, successfully generating CAD code from real-world images and executing CAD operations unseen during fine-tuning. The performance and adaptability of \ours{} highlights the potential of VLMs fine-tuned on code to streamline CAD workflows for engineers and designers. \ours{} is publicly available at: \href{https://github.com/anniedoris/CAD-Coder}{https://github.com/anniedoris/CAD-Coder}.
\end{abstract}

\begin{nomenclature}
\entry{LLM}{Large Language Model}
\entry{VLM}{Vision-Language Model}
\end{nomenclature}

\section{INTRODUCTION}
Generative AI has recently shown great promise in engineering design, with large language models (LLMs) and vision-language models (VLMs) beginning to automate parts of the engineering design process~\cite{alrashedy2024generating}.
An indispensable step in the design of any engineering system, consumer product, or structure is the creation of a 3D Computer-Aided Design (CAD) model. 
With the exception of specific functionalities, like generative design and topology optimization, the process of CAD modeling -- drawing sketches, defining constraints, extruding bodies -- remains largely manual, performed by designers with years of engineering experience and familiarity with CAD modeling software. 
This has motivated researchers to seek ways to partially automate CAD creation using learning-based approaches. 
An AI assistant capable of generating editable CAD models from textual or visual specifications would significantly reduce the time engineers spend on manual tasks. It would also make CAD modeling more accessible to beginners.
An important goal of the research has been to develop generative models that can produce valid, parametric CAD models from high-level inputs, thereby accelerating the design-to-production pipeline. Unlike most 3D shape generation work that outputs meshes or voxels, a CAD automation system should provide \emph{modifiable} designs in a programmatic form, since practical engineering applications demand editability and reusability of the 3D CAD.

\begin{figure*}[t]
\begin{center}
\setlength{\unitlength}{0.012500in}%
\includegraphics[width=\linewidth]{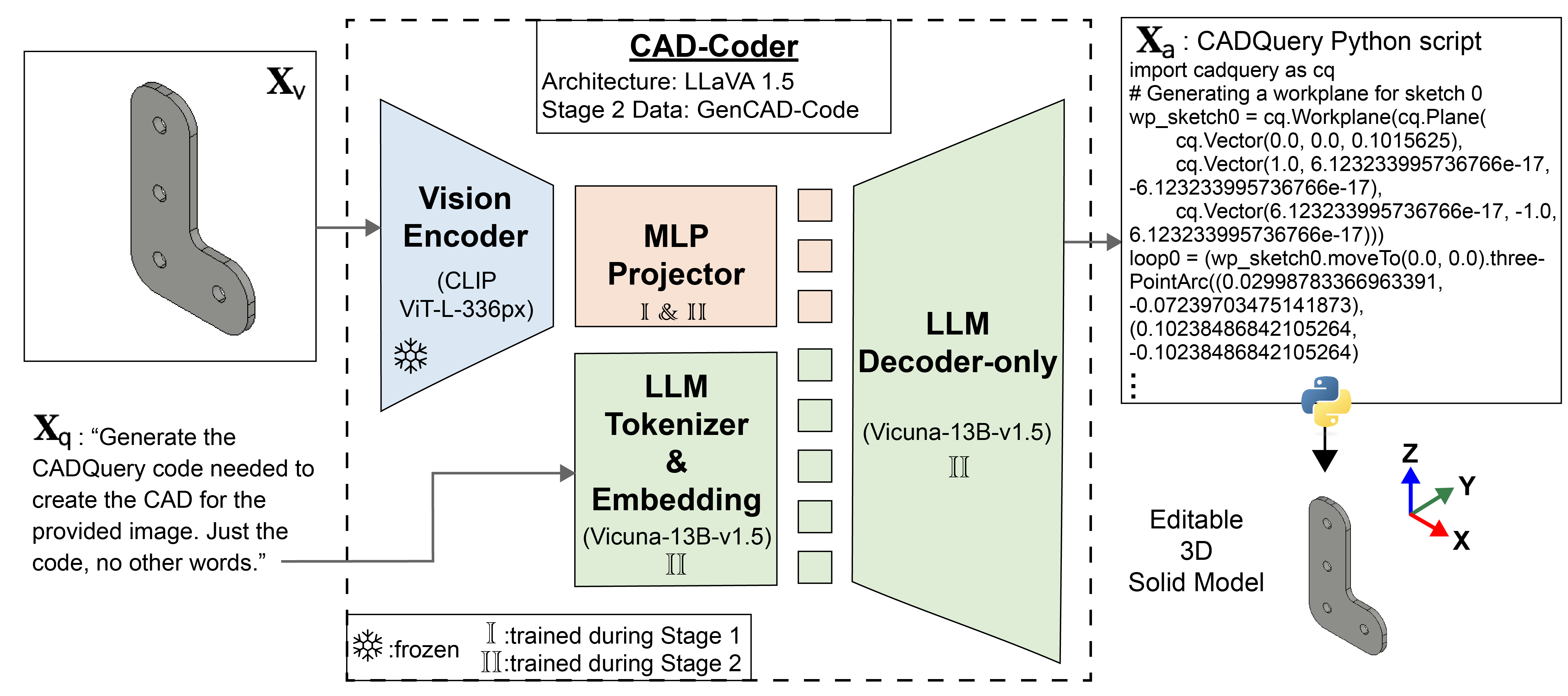}
\end{center}
\caption{Overview of \ours{}. The VLM accepts an image as input and outputs CadQuery code, which can be run as a Python script to produce an editable, 3D solid CAD model. \ours{} has a LLaVA 1.5-type architecture and is fine-tuned on the \dataset{}.}
\label{fig:overview} 
\end{figure*}

Early approaches to AI-driven CAD generation have made progress but face notable limitations. Many prior works trained bespoke models solely on CAD data or limited synthetic distributions, which restricted their scope of understanding. A number of works in recent years have explicitly trained ML models from scratch to generate CAD, many focusing non-editable 3D solid generation~\cite{zhou20213d, luo2021diffusion, jun2023shap, liu2022towards, michel2022text2mesh} with fewer focusing on editable, parametric CAD program generation~\cite{xu2022skexgen, wu2021deepcad, alam2024gencad}. Many of these past editable CAD systems focused on a small set of primitive operations, which limited their ability to handle real-world variation. CAD-specific models that are trained from scratch lack the broad visual and contextual knowledge that modern foundation models possess. As a result, prior trained-from-scratch CAD generators face two generalizability issues: 1) they may fail on inputs that deviate from their training distribution and 2) they fail when asked to use CAD operations not seen during training. As many existing AI solutions for 3D modeling either produce non-editable outputs (e.g., meshes) or rely on specialized training with limited generalization capacity, there is a need for a more robust approach to editable CAD generation.

Meanwhile, vision-language foundation models (VLMs) have achieved remarkable generalization across tasks in computer vision and natural language processing (NLP). VLMs, large language models (LLMs) integrated with vision encoders—such as GPT-4o~\footnote{https://openai.com/index/gpt-4o-system-card/} (closed-source) or LLaVA~\cite{liu2023llava} (open-source)—have the potential to mitigate these two issues. Pre-trained on billions of tokens, LLMs are already capable of generating code and therefore have the capacity to output a complete, editable CAD representation: code for CAD, using an existing, well-defined CAD scripting library. Pre-trained on millions of images, VLMs have learned meaningful latent representations of real-world images, positioning them well for CAD generation conditioned on images of real objects. 

Prior studies have used prompt engineering or few-shot strategies to coax general VLMs into image-conditioned CAD generation tasks, but without domain-specific training, their performance and reliability remains limited (e.g., they may hallucinate incorrect geometry or produce code with syntax errors~\cite{li2024llm4cad}). Fine-tuning has been used to significantly improve the performance of LLMs for text-conditioned CAD code generation~\cite{du2024blenderllm} and for image-conditioned generation of CAD programs expressed in domain-specific languages (DSLs)~\cite{wu2024cadvlm, yuan2024openecad}. However, reliance on DSL-based CAD programs--rather than established, full-featured CAD scripting libraries--restricts CAD generation to the narrow set of operations defined in the DSL and reduces a model's ability to leverage its pretraining knowledge.

This raises the question: can we harness a generalist VLM to generate \emph{parametric CAD code} conditioned on images, bridging the gap between high-level visual input and precise engineering code? To our knowledge, no published work has fine-tuned an open-source VLM specifically for the direct synthesis of CAD code using a widely-adopted scripting library. There is a gap in the literature for an approach that combines the broad CAD scripting and visual pre-training knowledge of foundation models with dedicated training on CAD code to achieve both versatility and accuracy in image-conditioned CAD scripting automation.

In this paper, we propose \textbf{\ours{}}, an open-source vision-language model tailored for computer-aided design code generation. We utilize a powerful, general VLM and fine-tune it end-to-end for the task of producing CAD code conditioned on image inputs. Concretely, \ours{} takes in an image and generates readily-executable code using CadQuery, a Python-based parametric CAD scripting library. We validate \ours{} through extensive experiments, demonstrating significant performance gains over the state-of-art, image-conditioned, CAD code-generating baselines. Notably, \ours{} attains an 100\% valid syntax rate, meaning every generated CAD script compiles successfully — a substantial improvement over the compile rates of existing solutions. In terms of generated 3D solid accuracy, \ours{}’s generated shapes closely match the ground-truth solids, improving upon the precision achieved by state-of-the-art baselines. We also conduct tests to probe generalization: for example, we feed \ours{} photographs of real objects (outside the fine-tuning dataset distribution) and prompt it to produce CAD code, or, we ask it to use CAD operations not seen during fine-tuning. In these challenging scenarios, \ours{} shows encouraging performance, often producing reasonable CAD where trained-from-scratch or DSL-based solutions would struggle.

In summary, our key contributions are as follows:

\begin{enumerate}
    \item \textbf{A VLM Trained for CAD Code Generation:} We introduce \ours{}, a fine-tuned, open-source vision-language foundation model designed specifically for the generation of CAD code. A LLaVA-derived VLM, our model, given image inputs, generates readily-executable Python CadQuery code. Our fine-tuned VLM approach is different from existing methods in that it is an image-conditioned method that outputs CAD in an \textit{established code representation} (i.e. CadQuery), addressing generalization limitations of prior works.

    \item \textbf{Performance Comparison}: \ours{} achieves superior results on the CAD code generation task, outperforming strong VLM baselines including GPT-4.5 and Qwen2.5-VL-72B. It produces valid code in 100\% of cases and significantly improves geometric accuracy of generated solids over previous methods, affirming our method of fine-tuning a foundation model for this task.
    
    \item \textbf{\ours{} Generalization Experiments}: We provide evidence that \ours{} has generalization ability to tasks not explicitly included in the model's fine-tuning dataset, owing to its inherited, broad knowledge. On a small test dataset of images of real, 3D-printed objects, we illustrate that \ours{} generates reasonable CadQuery code, an advantage over trained-from-scratch CAD generation approaches which struggle with images out-of-distribution with respect to their training data. We also illustrate that with careful prompting and selection of the pre-trained LLM that is fine-tuned, \ours{} variants can utilize CAD operations not explicitly included in the fine-tuning dataset, a benefit over existing trained-from-scratch approaches and methods reliant on DSL-based CAD programs.

    \item \textbf{A Publicly Available Dataset of CAD Code Paired with Images}: We generate CadQuery Python code for 163,671 CAD models. These Python files accompany the rendered CAD images from the GenCAD~\cite{alam2024gencad} dataset, comprising the largest publicly available dataset, named GenCAD-Code, of CAD code paired with CAD images.
\end{enumerate}

Through these contributions, our work positions foundation models as a promising avenue for next-generation CAD tools. \ours{} not only bridges the gap between high-level visual understanding and low-level CAD programming, but also underscores the value of marrying general VLMs with domain-specific fine-tuning to meet the requirements of engineering design. The following sections detail the methodology of \ours{}, its experimental evaluation against existing baselines, and a discussion of \ours{}'s generalizability to tasks not seen during fine-tuning.

\begin{table*}[t]
\caption{Overview of some existing approaches for conditional, editable-CAD generation.}
\label{tab:related_work}
\begin{center}
\setlength\tabcolsep{5pt} %
\renewcommand{\arraystretch}{1.2} %
\begin{tabular}{@{}lccccc@{}}
\toprule
Method & \makecell{Conditional \\ Input} & \makecell{Output: Editable- \\ CAD Representation} & 2D or 3D CAD & \makecell{Model\\Architecture} & \makecell{Training \\ Method} \\
\toprule
GenCAD~\cite{alam2024gencad} & Image & CAD program & 3D & \makecell{Transformer encoder/decoder + \\ Contrastive image encoder + \\ Diffusion prior} & From scratch \\
\midrule
CSGNet~\cite{sharma2018csgnet} & Image & CAD program & 2D/3D & \makecell{Encoder (CNN) \\ Decoder (RNN)} & From scratch \\
\midrule
Text2CAD~\cite{khan2024text2cad} & Text & CAD program & 3D & \makecell{Transformer} & From scratch \\
\midrule
CADCodeVerify~\cite{alrashedygenerating} & Text & \makecell{CAD Code: \\ CadQuery} & 3D & \makecell{GPT-4 \\ Gemini 1.5 Pro \\ CodeLlama} & \makecell{None: \\ Multi-Model} \\
\midrule
IdeaToCAD~\cite{ocker2025idea} & Image + Text & \makecell{CAD Code: \\ CadQuery} & 3D & \makecell{GPT-4o} & \makecell{None: \\ Multi-Agent} \\
\midrule
CAD-Assistant~\cite{mallis2024cad} & Image + Text & \makecell{CAD Code: \\ FreeCAD} & 3D & \makecell{GPT-4o} & \makecell{None: \\ Tool Augmented} \\
\midrule
LLM4CAD~\cite{li2024llm4cad} & Image + Text & \makecell{CAD Code: \\ CadQuery} & 3D & GPT-4/GPT-4V + Debugger & \makecell{None: \\ Debugger} \\
\midrule
BlenderLLM~\cite{du2024blenderllm} & Text & \makecell{CAD Code: \\ Blender} & 3D & \makecell{Qwen2.5-Coder-7B-Instruct} & \makecell{Fine-tuning + \\ Self-improvement} \\
\midrule
Text2CAD~\cite{sun2025large} & Text & \makecell{CAD Code: \\ CadQuery} & 3D & \makecell{GPT-3.5 based} & Fine-tuning \\
\midrule
Cad-VLM~\cite{wu2024cadvlm} & Image + Text & CAD program & 2D & \makecell{Asymmetric \\ transformer \\ encoder-decoder} & Fine-tuning \\
\midrule
OpenECAD~\cite{yuan2024openecad} & Image + Text & CAD program & 3D & \makecell{TinyLLaVA} & Fine-tuning \\
\midrule
\makecell[l]{\ours{} \\ (Ours)} & Image + Text & \makecell{CAD Code: \\ CadQuery} & 3D & LLaVA 1.5 & Fine-tuning \\
\bottomrule
\end{tabular}
\end{center}
\end{table*}

\section{RELATED WORK}
Since a number of works -- particularly in recent years -- have developed AI-based solutions for CAD generation, we narrow our focus to those existing methods that support a conditional input and generate \textit{editable} CAD representations as output (Table \ref{tab:related_work}). 

\subsection{Editable CAD Generation: Models \& Training}
The majority of the image and/or text conditioned editable-CAD generation models presented in Table \ref{tab:related_work} make use of transformers~\cite{vaswani2017attention}, which enable prediction of the next CAD operation based on the previously predicted operations. Model architectures differ in accordance with whether the model was trained from scratch or makes use of pre-trained LLMs and/or VLMs.

\subsubsection{Training from Scratch}
Several works \cite{alam2024gencad, sharma2018csgnet, khan2024text2cad} have trained conditioned, editable CAD generating models from scratch. For example, \cite{alam2024gencad} trained an image-to-CAD-program model -- consisting of three components -- from scratch. They first trained a transformer-based encoder-decoder to learn the latent representation of the CAD commands and to autoregressively generate CAD commands. Using the frozen CAD command encoder, they contrastively trained an image encoder to learn a joint image-CAD command latent space. Finally, they train a diffusion prior model to generate CAD command latent vectors conditioned on image latent vectors. Sequentially combining the contrastively trained image encoder with the diffusion prior model with the autoregressively trained transformer decoder yields their model, GenCAD, which can generate a CAD command sequence given an input image. GenCAD achieved SOTA -- outperforming DeepCAD~\cite{wu2021deepcad} and SkexGen~\cite{xu2022skexgen} -- on chamfer-distance-based solid accuracy metrics.

\subsubsection{Leveraging Foundation Models}
Instead of training a model from scratch, researchers have begun utilizing LLMs and VLMs as a foundation for CAD generation and design automation. A key advantage of LLMs is their ability to handle free-form user inputs and generalize across tasks, making them attractive for engineering design applications. 
\paragraph{Without Additional Training:} A number of works have explored methods for using off-the-shelf pre-trained foundation models, including leveraging multiple models~\cite{alrashedygenerating, ocker2025idea} or integrating the models with external tools~\cite{du2024blenderllm, mallis2024cad}. Alrashedy et al.~\cite{alrashedygenerating} argue that a single-pass LLM often fails to meet all design requirements, so they incorporate a VLM in the loop to iteratively verify and correct the output. In their CADCodeVerify framework, GPT-4 first writes CAD code given a prompt, then a VLM analyzes the rendered 3D solid and asks itself targeted questions (e.g., “Are the holes the correct diameter?”). The VLM’s answers help identify deviations, which are fed back to GPT-4 to refine the code. This self-critique loop improved the shape accuracy and success rates of generated models. Another work~\cite{ocker2025idea} models the design process as a team of specialized AI agents and realize a human-AI co-design loop with multiple LLM-based agents taking on different roles (e.g., requirement engineer, CAD engineer, quality checker). This multi-agent approach mirrors how human design teams iterate on CAD, and was shown to produce more complete and error-checked CAD outputs than a zero-shot VLM.

CAD-Assistant by Li et al.~\cite{mallis2024cad} integrate a VLM ``planner'' with a Python-API-driven CAD engine (FreeCAD) to enable zero-shot solving of generic CAD tasks. \cite{du2024blenderllm} combined GPT-4 and GPT-4V with a debugger so that model-generated programs with syntax errors could be iteratively improved. Despite the power and extensive pre-training of OpenAI's GPT models, GPT-4 and GPT-4V were in many cases incapable of generating accurate CAD code. CAD model accuracy (IOU) dropped to near zero for mechanical spring CAD models, even when GPT-4 was integrated with the debugger. 

\paragraph{With Fine-Tuning:} The previously highlighted studies underscore the growing role of foundation models in CAD and engineering design. However, without additional training, there is a limit to how much a model's performance can improve on the CAD generation task. Fine-tuning, where all or some portion of a model's weights are further trained on the desired task, offers a middle-ground between training from scratch and leveraging off-the-shelf pre-trained models. Relatively few studies—\cite{du2024blenderllm}, \cite{sun2025large}, \cite{wu2024cadvlm}, and \cite{yuan2024openecad}—have investigated foundation models' capacity to be fine-tuned for the CAD generation task. \cite{du2024blenderllm} perform end-to-end supervised fine-tuning on an LLM to improve its ability to generate Blender Python code. Finetuning results in 2X improvement in the accuracy of generated 3D solids and in a substantial reduction ($\sim$30\%) in the number of generated scripts with syntax errors. Their fine-tuned model also substantially outperforms GPT-4o and other closed-source LLMs in its ability to generate Blender CAD code. Similarly, \cite{sun2025large} fine-tunes GPT-3.5 on a specialized text-to-CadQuery dataset, significantly improving success rates of code generation. Extending to the visual domain, \cite{wu2024cadvlm} fine-tune a pre-trained vision transformer on a dataset of sketches to improve its ability to generate 2D CAD. Yuan et al. \cite{yuan2024openecad} use LoRA to train OpenECAD, a fine-tuned TinyLLaVA model that outputs CAD programs given an image input, showing that their method outperforms gpt-4o-mini on an accuracy-based scoring function. The authors also illustrate the generalization potential of fine-tuned VLMs, by demonstrating that OpenECAD can generate a CAD program when conditioned on a hand-drawn sketch (despite the fine-tuning dataset only including images of CAD.)

A fine-tuned foundation model could offer several advantages over training a CAD generation model from scratch. Pre-trained models have already learned skills that are translatable to the CAD generation task. When an existing, code-based CAD representation is used (e.g. CadQuery), fine-tuning a pre-trained LLM that has pre-existing knowledge of that code could enable the model to leverage pre-training knowledge and utilize CAD operations that may have not been present in a limited fine-tuning dataset. Likewise, if image-conditioned CAD generation is desired, utilizing a pre-trained image encoder (e.g. CLIP~\cite{radford2021learning}) that has been trained on millions of image-text pairs could improve the model's ability to generate a meaningful image latent vector. Even more critically, pre-trained image encoders -- which have been trained on images of real-world objects -- can promote CAD generation conditioned on images of \textit{real-world} objects, when most existing image-CAD datasets~\cite{alam2024gencad} available for training are limited to images of rendered CAD. Our work continues along this line by fine-tuning a vision-language foundation model specifically for CAD code synthesis.

\subsection{Conditional CAD Generation}
Conditional methods are those that base their CAD generation on an input -- typically text, an image, or both -- so that the generated CAD is derived from a user input. For mechanical engineering tasks (e.g., creating CAD from a physical prototype), conditional CAD generation is much more useful than unconditional generation, where CAD is generated by sampling from a learned distribution. Until recently, much existing work on AI for CAD has focused on unconditional generation~\cite{wu2021deepcad, xu2022skexgen, willis2021engineering}.

All of the works presented in Table \ref{tab:related_work} explore methods for input-conditional CAD generation. \cite{du2024blenderllm, sun2025large, khan2024text2cad, alrashedygenerating} focus on text-conditioned CAD generation. For example, Khan et al.~\cite{khan2024text2cad} present a direct text-to-parametric CAD framework that accepts text instructions and outputs the corresponding CAD program. While text-conditioned CAD generation allows for user control over the generated CAD, geometries of real-world objects can be difficult to express using natural language. \cite{sharma2018csgnet, alam2024gencad} develop models explicitly trained to generate a CAD program given an input image. Combining modalities, \cite{wu2024cadvlm, ocker2025idea, mallis2024cad, li2024llm4cad, yuan2024openecad} all develop methods for image + text conditioned CAD generation. \cite{li2024llm4cad} go as far as to study the differences between conditioning VLMs on text versus image + text for CAD generation. While text-only conditioning was surprisingly effective, they found that providing both image and text inputs increasingly paid off for more complex solids. Building on these works and findings, \ours{} is conditioned on both input images and text, although, for now, we employ uniform text prompts.

\subsection{Editable CAD Representations and Datasets}
We also focus on methods and datasets that make use of ``editable'' CAD representations, parametric CAD programs that encode the design history of the geometric model. These editable representations are more useful to mechanical designers than other representations of 3D solids -- such as voxels~\cite{zhou20213d}, point clouds~\cite{luo2021diffusion}, or meshes~\cite{jun2023shap} -- since parametric values can easily be modified after generation.

Refs. \cite{wu2021deepcad, alam2024gencad, sharma2018csgnet, wu2024cadvlm, yuan2024openecad} train models that output editable CAD in the form of DSL-based CAD programs. These CAD programs treat CAD commands as a language or as a grammar. For example, \cite{khan2024text2cad} and \cite{alam2024gencad} use the CAD program representation defined by \cite{wu2021deepcad}, where a CAD program is represented as a sequence of CAD command-parameter vectors, with individual commands and parameters comprising a vocabulary. Limited to 2D, \cite{wu2024cadvlm} use a different, but similarly structured, CAD program representation. Ref. \cite{yuan2024openecad} also define their own DSL for CAD commands (OpenECAD operations) which they later convert into PythonOCC code for execution. These CAD program representations are challenging because they cannot generate solids without conversion scripts that translate them into code-based formats, and they are often incomplete. Ref. \cite{wu2024cadvlm}'s CAD program representation is limited to sketches (lines, arcs, and circles) plus constraints while \cite{wu2021deepcad}'s CAD program representation is limited to the same sketch commands (defined differently) plus extrusion commands. As of yet, no complete -- including splines, revolves, sweeps, lofts, fillet, chamfer, etc. --  CAD program representation has been defined. Defining a complete DSL-based CAD program representation would be challenging due to the non-linear interdependence and referential nature of many CAD operations. 

An alternative to defining a CAD program representation is to utilize an existing, CAD code scripting package. Written in widely-adopted coding languages (e.g. Python, C++), this CAD code representation aligns well with the knowledge base of foundation models. These scripting packages have the benefit of inherently being editable while also offering a complete representation, as they can be used to define any 3D CAD model. \cite{ocker2025idea, du2024blenderllm, li2024llm4cad, alrashedygenerating, sun2025large, mallis2024cad} leverage pre-trained models to output code that can generate CAD using the Blender API,~\footnote{https://docs.blender.org/api/current/} the CadQuery Python package,~\footnote{https://CadQuery.readthedocs.io/en/latest/}, or the FreeCAD Python API.~\footnote{https://wiki.freecad.org/FreeCAD_API} The Blender API only works in conjunction with the Blender application/GUI, which is primarily designed for mesh-based animations and graphics. CadQuery and FreeCAD are both wrappers around OpenCascade,~\footnote{https://dev.opencascade.org/} a powerful, B-rep-based geometric modeling kernel and both can be run as stand-alone, GUI-less Python scripts for parametric CAD generation, an advantage over the Blender API. Given CadQuery's popularity~\cite{li2024llm4cad, sun2025large, alrashedygenerating, ocker2025idea} and extensive documentation, we train \ours{} to output CadQuery code. Note that in this paper, all references to ``CAD code'' refer to an editable CAD representation that makes use of a well-established CAD scripting library (e.g. CadQuery), while ``CAD program'' refers to a commands-as-vocabulary or author-defined DSL representation, such as those used by \cite{wu2021deepcad, wu2024cadvlm, alam2024gencad, yuan2024openecad}, discussed previously.

The availability of large scale, editable CAD datasets is somewhat limited due to the difficulty of obtaining human-created designs from commercial 3D programs. One of the largest available CAD datasets, ABC dataset~\cite{koch2019abc}, contains $1$M 3D models in different data formats such as STEP, STL, domain specific language (FeatureScript), and image. However, the ABC dataset includes duplicates and designs that do not produce valid 3D shapes~\cite{xu2024brepgen, jayaraman2022solidgen}. Many of the available editable CAD datasets in literature are variants of the ABC dataset. For example, the DeepCAD dataset~\cite{wu2021deepcad} contains approximately $172k$ CAD programs which are derived from a selected subset -- those containing prismatic sketch and extrude operations -- of 3D solids in the ABC dataset. \cite{khan2024text2cad} build on the DeepCAD dataset by pairing the $172k$ CAD models with 660K textual annotations. The GenCAD dataset~\cite{alam2024gencad} also builds on the DeepCAD dataset by coupling each of 168k CAD programs with five grayscale, isometric view images (of varying scales) of the rendered CAD. Our dataset, \datasetnoe{}, builds on the GenCAD dataset by converting all of the CAD programs into CadQuery code. Some existing CAD code datasets exist, although they are limited in size. \cite{du2024blenderllm} generated a dataset of 8k Blender scripts, while \cite{li2024llm4cad} created a custom dataset of 5 classes of mechanical parts (e.g., gears, springs) comprised of 5k examples, each consisting of a sketch, image, text, and CadQuery script. Our dataset, \datasetnoe{}—comprised of 163k image-CadQuery pairs—significantly expands the scale of image-CAD code datasets.

\section{METHODS}
In this section, we describe the generation of our \dataset{} and the architecture and training of \ours{}. We also characterize the evaluation metrics we use and the baselines we select for comparison with \ours{}.

\subsection{Dataset Generation}
\label{sec:data_gen}
To generate our \dataset{}, we converted each CAD program in the GenCAD dataset~\cite{alam2024gencad} to a CadQuery script. Each GenCAD CAD program consists of a sequence of CAD commands, and each CAD command is represented by $\mathbf{c}_i = (t_i, \mathbf{p}_i)$, where $\mathbf{c}_i \in \mathbb{R}^{17}$, $t_i$ is the command type ($t_i \in {\{\text{Sketch Line, Sketch Arc, Sketch Circle, Extrude}}\}$), and $\mathbf{p_i} \in \mathbb{R}^{16}$ are the parameters associated with each command. We write a script that directly converts each of these CAD commands into its corresponding CadQuery code; additional details can be found in our \href{https://github.com/anniedoris/CAD-Coder}{repository}. The GenCAD dataset provides five CAD rendered images per CAD program: one isometric view and four scaled isometric views. This process resulted in our \dataset{} of 163,671 CadQuery scripts, each matched with five CAD rendered images. There are 147,289 samples in the training set, 7,355 in the test set, and 9,027 in the validation set. 

The distribution of token counts for all 163k generated CadQuery scripts can be seen in Figure \ref{fig:token_count}. The average number of tokens per CadQuery script is 611 tokens, and 99.9\% of the CadQuery scripts contain less than 3000 tokens. In general, 3D solid complexity corresponds with number of tokens in its CadQuery script (see Figure \ref{fig:token_count}). The right-skewed dataset distribution illustrates that the dataset is imbalanced when it comes to simple and complex examples. Future work will focus on improving dataset balance to include more ``complex'' examples. 

\begin{figure}[t]
\begin{center}
\setlength{\unitlength}{0.012500in}%
\includegraphics[width=\columnwidth]{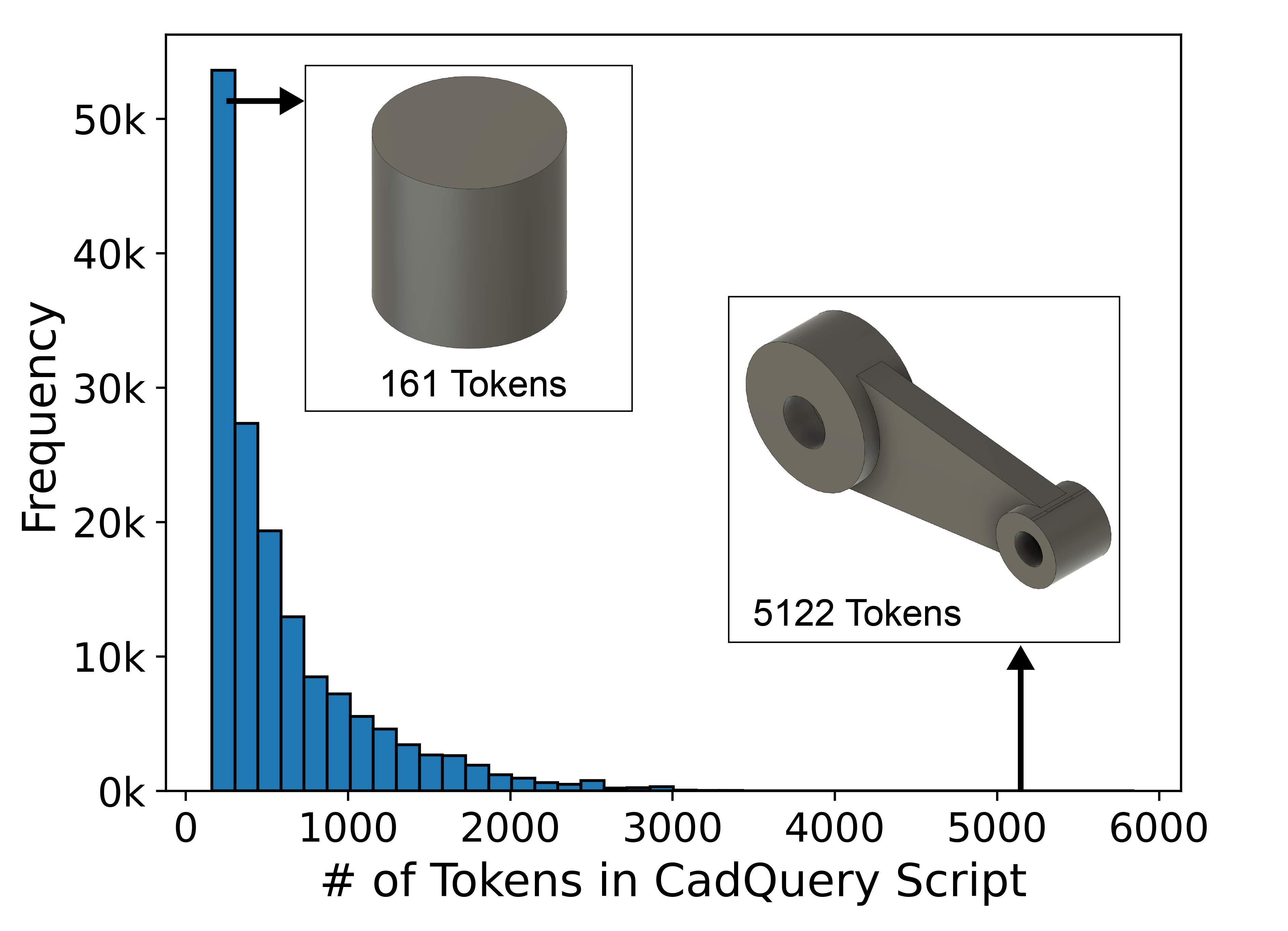}
\end{center}
\caption{Distribution of token counts for the CadQuery scripts in our \dataset{}.}
\label{fig:token_count} 
\end{figure}

It is important to note that the direct conversion process we use from the GenCAD CAD programs to CadQuery scripts does not necessarily result in the most concise CadQuery code. For example, the GenCAD CAD programs support only three sketch operations: line, arc, and circle. These CAD programs define rectangle sketches using four sequential line commands, and as such, our generated CadQuery scripts do as well. However, there is a more concise way in CadQuery to generate rectangle sketches, namely the \texttt{.rect()} method. %
Future work will explore methods for post-processing CadQuery scripts to leverage CadQuery's native concise representations (e.g., replacing sequential line commands with more efficient methods like rect()), thereby reducing script complexity.

\subsection{\ours{} Model Architecture and Training}
To train \ours{}, we use the same architecture as LLaVA 1.5~\cite{liu2024visual} and follow a two-stage training process similar to the visual instruction tuning method used to train LLaVA~\cite{liu2023llava, liu2023improved}. This method synthesizes a pre-trained LLM with a pre-trained vision encoder. We use Vicuna-13B-v1.5, a 13 billion parameter LLM model, as our pre-trained LLM and CLIP-ViT-L-336px as our vision encoder. CLIP-ViT operates on higher resolution images (336x336 pixels) than other CLIP variants, enabling it to capture finer image details.

\paragraph{Stage 1: Pre-training for Feature Alignment}
The function of the pre-training phase is to learn a mapping from the image features of the vision encoder to the word embeddings of the LLM. This is done by training a two-layer multi-layer perceptron (MLP) while keeping the vision encoder and LLM weights frozen. The dataset used for this pre-training phase is identical to that used by \cite{liu2023llava} and consists of 595k samples, each consisting of a question (text-only, $\mathbf{X}_{\texttt{q}}$) accompanied by an image ($\mathbf{X}_{\texttt{v}}$) and an answer (text-only, $\mathbf{X}_{\texttt{a}}$). This dataset was generated by~\cite{liu2023llava} by taking 595K image-caption pairs from the CC3M dataset, where $\mathbf{X}_{\texttt{v}}$ is the image, $\mathbf{X}_{\texttt{a}}$ is the original caption, and $\mathbf{X}_{\texttt{q}}$ is randomly selected from a finite set of questions that asks, in some way, for a description of the image (e.g. $\mathbf{X}_{\texttt{q}}$ = "Give a brief description of the image."). The two-layer MLP is trained to maximize the autoregressive likelihood function:

\begin{equation}
\label{eqn:autoregressive}
    p(\mathbf{X}_{\texttt{a}} \mid \mathbf{X}_{\texttt{v}}, \mathbf{X}_{\texttt{q}}) = \prod_{i=1}^{L} p_{\theta} (x_i \mid \mathbf{X}_{\texttt{v}}, \mathbf{X}_{\texttt{q}, < i}, \mathbf{X}_{\texttt{a}, < i})
\end{equation}

where $L$ is the tokenized length of $\mathbf{X}_{\texttt{a}}$, $\theta$ are the trainable parameters (in this case the weights of the two-layer MLP), and $\mathbf{X}_{\texttt{q}, < i}$ and $\mathbf{X}_{\texttt{a}, < i}$ are the question and answer tokens before the current prediction token, $x_i$.

\subsubsection{Stage 2: End-to-End \datasetnoe{} Fine-tuning}
\label{sec:e-to-e}
After alignment of the vision encoder and word embeddings during pre-training, we conducted supervised fine-tuning. To fine-tune, we use the 147k \datasetnoe{} training dataset described in Section \ref{sec:data_gen}. For each \datasetnoe{} pair, $\mathbf{X}_{\texttt{a}}$ is the CadQuery code, $\mathbf{X}_{\texttt{v}}$ is the unscaled isometric CAD rendered image, and $\mathbf{X}_{\texttt{q}}$ is always constant as: ``Generate the CadQuery code needed to create the CAD for the provided image. Just the code, no other words.'' We leave variation of the text prompt as future work. We filtered out training samples (less than 0.1\% of the original training dataset) that would cause the total number of tokens ($\mathbf{X}_{\texttt{q}}$ + $\mathbf{X}_{\texttt{v}}$ + $\mathbf{X}_{\texttt{a}}$) to be greater than 4096, the maximum number of tokens supported by many LLMs. With the vision encoder weights frozen, we again maximize the autoregressive likelihood function (Equation \ref{eqn:autoregressive}), except this time, the trainable parameters, $\theta$, are \textit{both} the two-layer MLP weights as well as all of the LLM's weights.

\subsection{Evaluation}
We compare \ours{} against state-of-the-art image-conditioned, CAD-code-generating baselines using two primary metrics. 
Due to the computational cost and inference time involved in generating and evaluating CAD programs, we evaluate all baselines on a representative subset of 100 randomly sampled test examples.
These samples are from the test set described in Section \ref{sec:data_gen} and formatted in the same way as the training data (see Section \ref{sec:e-to-e}).

\subsubsection{Baselines}
We compare our models to both closed-source and open-source VLMs capable of generating CadQuery code. For closed-source VLMs, we evaluate the top-three, best-performing\footnote{As of March 17, 2025.} models on the MMMU benchmark leaderboard~\cite{yue2024mmmu}: GPT-o1 (OpenAI's multimodal reasoning model), GPT-4.5 (OpenAI's advanced multimodal language model), and Gemini-2.0-Pro (Google's latest multimodal model). For open-source VLMs, we evaluate the top-three, best-performing\footnote{As of March 17, 2025.} models on the Huggingface OpenVLM Leaderboard~\cite{duan2024vlmevalkit}: InternVL2\_5-78B-MPO (Shanghai AI Lab's large-scale open-source VLM), Ovis2-34B (OpenGVLab's open-source multimodal model), and Qwen2.5-VL-72B (Alibaba's state-of-the-art VLM). To demonstrate the effectiveness of our training data and method for image-conditioned CAD-code generation, we also evaluate LLaVA-v1.5-13B, which is trained identically to our models except stage 2 uses a generic VQA dataset rather than our domain-specialized \dataset{}.

All of the baselines were tested using their default inference parameters except for maximum output tokens, which was set to 3450. This number was used so that the total number of tokens (accounting for the fixed-length prompt and image tokens) was 4096, the maximum number possible for many LLMs. An exception was made for GPT-o1 (max\_completion\_tokens was set to its default value), since its maximum output token quota is used by both the final output and the reasoning process, which happens under-the-hood and a user has no control on it. Since these models were not trained on our \dataset{} -- which ends each CadQuery script by returning the final solid assigned to the variable ``solid'' -- we appended to the prompt: ``Assign the final solid to the variable `solid' in the last line of code. Do not export or visualize the solid.'' This addition facilitates automated evaluation of the baselines' generated code. For Qwen2.5-VL-72B, we noticed that it often neglected to include package import statements in its generated scripts, so we also appended to the prompt: ``Be sure to include necessary imports.'' To encourage reproducibility of our results, we evaluate \ours{} with a temperature of 0 so that the VLM always selects the most probable next token. For reproducibility, we provide all trained model weights and the exact prompts used during evaluation.

\subsubsection{Metrics}
We evaluate all of the baselines and \ours{} on two metrics: valid syntax rate (VSR) and the intersection-over-union corresponding with the best alignment between the model-generated solid and the ground-truth solid (IOU\textsubscript{best}).

\paragraph{Valid Syntax Rate (VSR):} This metric is defined as the percentage of model-generated CadQuery scripts from the test subset that are syntactically valid and return no errors when run as Python files.

\paragraph{IOU\textsubscript{best}:}
To compare the shapes of solid geometries, denoted as $\Omega$ here, we first normalize the translation and scale of the models using a normalization operator:

\begin{equation}
    n(\Omega)=\{\frac{\mathbf{x}-\bar{\mathbf{x}}}{\sqrt{\frac{\operatorname{tr}(\mathbf{I})} {2\times \mathrm{V o l} ( \Omega_{2} )}}}\mid\mathbf{x}\in \Omega\},
\end{equation}

where $\bf I$ is the matrix of inertia and $\bf \bar{x}$ is the centroid for the solid $\Omega$. This normalization scales the object by the root mean squared radius of gyration, a choice justified and proven to be optimal in Appendix~\ref{app:iou} and Lemma~\ref{lm:aligment}. Note that in this paper, the goal is to generate correct geometry with respect to images. As such, scale information is not inherently an input to the models, and therefore, one must normalize scales for comparison. After this normalization, we align the principle axes of the generated and ground truth solids. The direction of principal axes is ambiguous, and as such we perform all possible valid alignments and use the one with the best Intersection Over Union~(IOU) value. The choice of these normalization parameters and the alignment proposed is justified in Appendix \ref{app:iou}. Unlike commonly used metrics such as Chamfer distance with bounding box or corner alignment (e.g., \cite{alam2024gencad,wu2021deepcad}), where shapes are converted to point clouds and normalized by translating the minimum to the origin and scaling the maximum to 1, our approach provides a higher fidelity and mathematically sound approach for comparing solid objects. While IOU-based measures have been used before to compare solid geometry~\cite{li2024llm4cad}, they have often been overlooked in prior work on datasets such as DeepCAD~\cite{jiang2022deep}, which relied on the aforementioned Chamfer distance. As demonstrated in Appendix~\ref{app:iou}, we argue that our approach is a more robust and principled alternative. Note that IOU\textsubscript{best} is computed over the set of CAD code scripts that execute successfully, excluding those with syntax errors.

\section{RESULTS}
Below we evaluate the performance of our proposed vision-language model (\ours{}) against baselines and additional variants, assessing both code validity (VSR) and geometric accuracy (IOU\textsubscript{best}) on the \datasetnoe{} test set.

\begin{table}[!th]
\centering
\caption{Evaluation of state-of-the-art, image-conditioned, code-generating models on 100 samples from the \datasetnoe{} test set. All models are evaluated on their rate of generating syntactically-valid CadQuery scripts (VSR) and on the accuracy of the resulting 3D solids (IOU\textsubscript{best}).}
\label{tab:main-results}
\resizebox{\columnwidth}{!}{%
\begin{tabular}{@{}lccc@{}}
\toprule
\textbf{Model} & $\quad \quad \quad \quad$ & \textbf{VSR $\uparrow$} & \textbf{IOU\textsubscript{best} $\uparrow$} \\ 
\midrule
\multicolumn{4}{c}{\cellcolor{gray!30}\textbf{Open Source Models}} \\
\midrule
InternVL2_5-78B-MPO & & 88\% & 0.379 \\ 
Ovis2-34B &  & 83\% & 0.408 \\ 
Qwen2.5-VL-72B &  & 94\% & 0.352 \\ 
LLaVA-v1.5-13B &  & 0\% & 0.0 \\
\midrule
\multicolumn{4}{c}{\cellcolor{gray!30}\textbf{Closed Source Models}} \\
\midrule
GPT-o1 &  & 88\% & 0.494 \\
GPT-4.5 &  & 84\% & 0.524 \\
Gemini-2.0-Pro &  & 82\% & 0.445 \\
\midrule
\multicolumn{4}{c}{\cellcolor{gray!30}\textbf{Ours}} \\
\midrule
\ours{} &  & \textbf{100\%} & \textbf{0.675} \\
\addlinespace
\bottomrule
\end{tabular}%
}%
\end{table}

\subsection{Training Details}
\ours{} is trained on 4 H100 GPUs. We follow the training parameters used for LLaVA 1.5~\cite{liu2023improved} with the exception of ``max\_model\_length'' which we set to 4096 for both stages (rather than 2048) to accommodate longer CadQuery scripts. Stage 1 training took 4.5 hours, and we trained for 1 epoch using a learning rate of 1e-3 and an effective batch size of 256. Stage 2 training took 5.7 hours, and we again trained for 1 epoch using a learning rate of 2e-5 and an effective batch size of 128. 

\subsection{Comparison with Baselines}
\label{sec:compare_with_baselines}
We present the results of the evaluation of \ours{} against existing image-conditioned, code-generating baselines in Table \ref{tab:main-results}. \ours{} outperforms all of the evaluated baselines in terms of both the valid syntax rate and the IOU metric. For each of the 100 samples in the test subset, \ours{} generates valid CadQuery code that throws no errors when run as Python scripts. The next-best performer in terms of VSR is Qwen2.5-VL-72B, which has a valid script generation rate of 94\%, although this is accompanied by an IOU\textsubscript{best} approximately half that of \ours{}.

In terms of IOU\textsubscript{best}, \ours{} -- with a score of 0.675 -- performs $\sim$60\% better than the next-best open-source model evaluated, Qwen2.5-VL-72B. \ours{} also outperforms all of the closed-source models evaluated, and the next-best performing model evaluated (GPT-4.5) has an IOU\textsubscript{best} score 0.151 lower. To help contextualize these IOU scores, we present various model-generated 3D solids and their corresponding IOU\textsubscript{best} scores in Figure \ref{fig:iou_context}. All of the models are shown in the alignment that maximizes their IOU\textsubscript{best} score. The disk models in the top row are organized from left-to-right by descending IOU\textsubscript{best} score. \ours{}'s generated solid -- with a near perfect score of 0.963 -- very closely matches the ground truth solid with the exception that the central hole is slightly larger. 
Even this small difference of 0.04 from a perfect IOU\textsubscript{best} score (1.0) is visually noticeable, indicating that \ours{}'s 0.151 improvement in IOU\textsubscript{best} score over the next-best-performing model is significant. The second row of Figure \ref{fig:iou_context} shows a more complex solid. While \ours{}'s generated solid is missing the rectangular feature in front, its double triangle shape better matches the ground truth solid than any of the baseline generated solids. This more complex example illustrates how many of the existing image-conditioned, code-generating baselines struggle as CAD model complexity increases. While \ours{} achieves a higher IOU\textsubscript{best} than the evaluated baselines, it still has significant room for improvement, particularly in generating CAD with more complex features.

It is also interesting to note that LLaVA-v1.5-13B -- which has the same architecture as \ours{} but is trained during stage 2 using a general-purpose VQA dataset rather than the \dataset{} -- has a score of zero on both VSR and IOU\textsubscript{best} metrics. While the model generates code-type text reminiscent of CadQuery, none of it is syntactically correct. This result stems from the fact that LLaVA-v1.5-13B—and its underlying LLM, Vicuna-13B-v1.5—seems to have no working knowledge of CadQuery. This demonstrates the usefulness of the architecture (LLaVA v1.5) and training strategy (domain-specific-only for stage 2) that we use for \ours{}. Our method produces a domain-specific state-of-the-art model, \textit{even when the pre-trained LLM used has no pre-existing knowledge of the domain-specific task}.

\begin{figure*}[!htb]
\begin{center}
\setlength{\unitlength}{0.012500in}%
\includegraphics[width=0.7\linewidth]{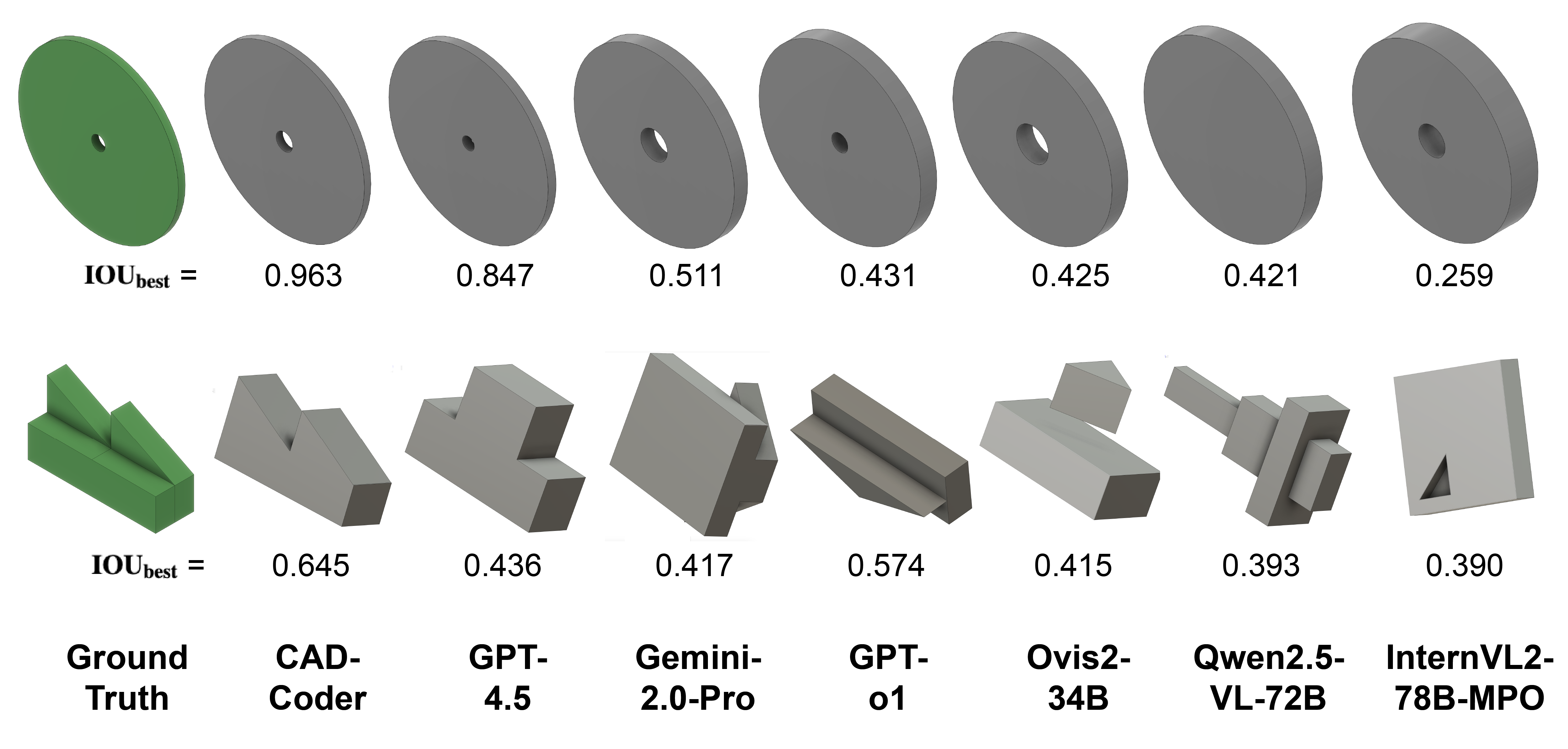}
\end{center}
\caption{Two examples comparing \ours{}'s generated solids with baseline generated solids. The IOU\textsubscript{best} score quantifies a solid's similarity to a ground truth solid, where an IOU\textsubscript{best} of 1 is a perfect score. The solids are depicted in their alignments that yield the IOU\textsubscript{best} score.}
\label{fig:iou_context} 
\end{figure*}

\begin{figure*}[t]
\begin{center}
\setlength{\unitlength}{0.012500in}%
\includegraphics[width=0.7\linewidth]{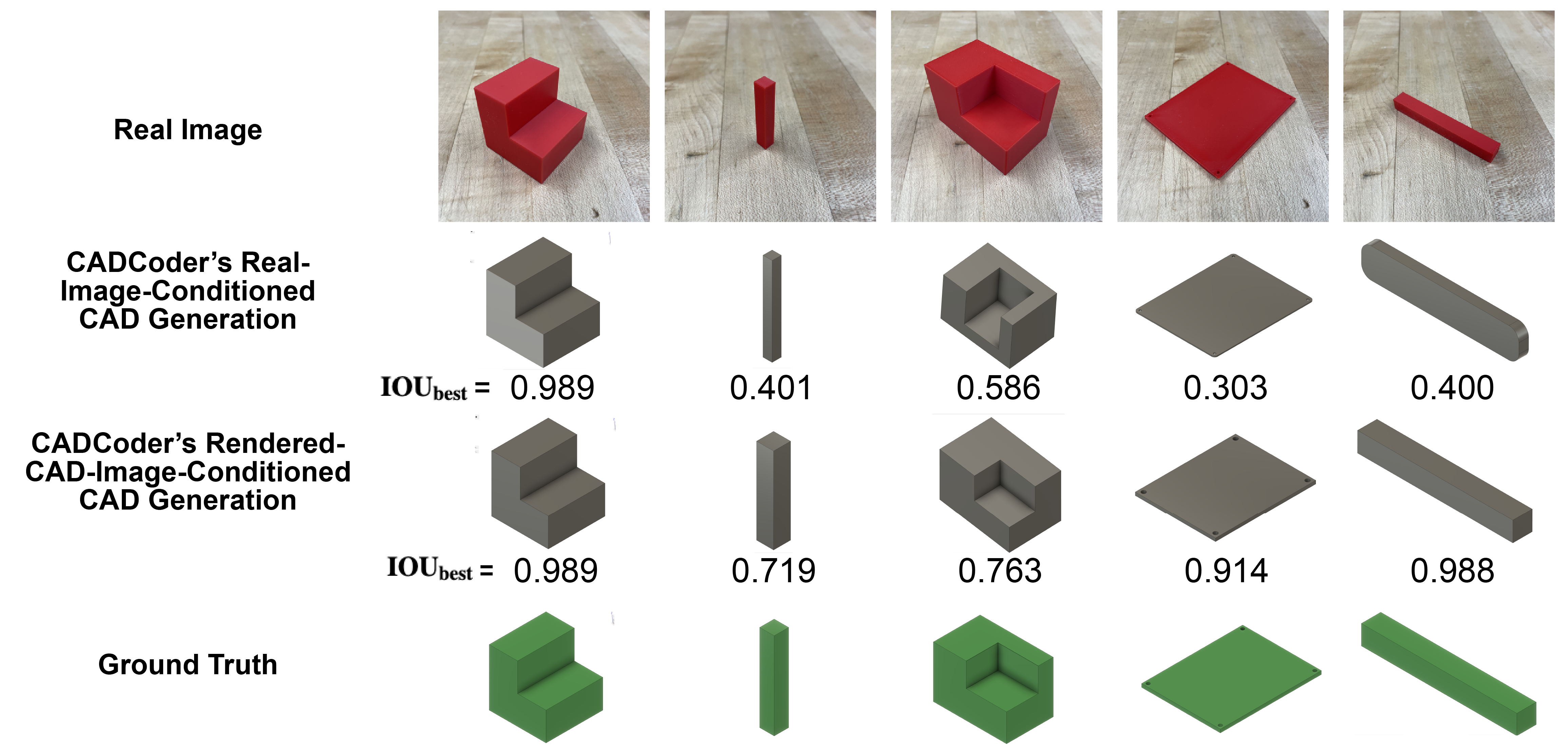}
\end{center}
\caption{We test \ours{}'s generalizability to real-image-conditioned CAD generation, a task not included in the fine-tuning dataset. 1st row: we 3D print several objects from \datasetnoe{}'s test set and photograph them at approximately isometric views. 2nd row: \ours{}'s real-image-conditioned CAD generation. 3rd row: \ours{}'s rendered-CAD image-conditioned CAD generation. 4th row: ground truth solids.}
\label{fig:real_object} 
\end{figure*}

\subsection{Generalization Experiments}
We hypothesized that a key benefit of training an image-conditioned, CAD generating model using VLMs fine-tuned on CAD code would be in the generalizability of the model. In this section, we investigate the generalizability of \ours{} to images of real-world objects and evaluate \ours{}'s extendability to CAD operations not explicitly included in the \dataset{}.

\subsubsection{Performance on Images of Real-World Objects}
While \ours{} was not explicitly fine-tuned to generate CAD code from images of real objects, we hypothesized that \ours{}’s use of a pre-trained image encoder (CLIP-ViT-L-336px) would help it generalize to this unseen task. To investigate its generalizability to real images, we 3D printed five objects from the test set of \dataset{}. We photographed the objects on a wooden table at approximately isometric angles (Figure \ref{fig:real_object}), as all of the \dataset{} rendered CAD images capture isometric views of the CAD solids. The 3D CAD models that \ours{} generates given these real-image inputs can be seen in the second row of Figure \ref{fig:real_object}. While not perfect matches, these generated CAD models conditioned on the real images broadly agree with the ground-truth 3D solids. These results indicate that \ours{} has potential – where trained-from-scratch models struggle – to generalize to CAD generation from real images, even when not explicitly fine-tuned on this task.

The importance of this rendered-CAD image to real-world image translation cannot be overstated. Developing a sufficiently large dataset of real images coupled with CAD code would be extremely expensive, necessitating the time-consuming and costly fabrication of hundreds of thousands of physical objects. As such, image-CAD code datasets will likely remain limited to images of rendered CAD models rather than real photographs for the foreseeable future. However, the typical user of an image-conditioned, CAD generating model will rarely input an image of a CAD model, as this implies that they already have the CAD. It is much more likely that a user would want to generate a CAD model given a real photograph. Quantifying the generalizability of CAD generating models to real images is therefore critical.

While \ours{} exhibits generalizability to real images, the model still performs noticeably better at generating CAD code given images of rendered CAD models, as seen by the relative score differences between the second and third rows of Figure \ref{fig:real_object}. In some of the real-image-conditioned generation examples (e.g. objects 2 and 4), \ours{} predicts the correct CAD operations, but a poor aspect ratio, resulting in a lower IOU\textsubscript{best} than the rendered-CAD image-conditioned equivalent. We can see this for the second object, where \ours{} underestimates the object’s side length. \ours{}’s decreased dimensional accuracy given images of real objects – \textit{colorful} images from slightly \textit{varying perspectives} –  is perhaps not surprising given that we train it on \textit{perfectly isometric} and \textit{gray} views. To mitigate this perspective problem, future work will train \ours{} on multiple material-rendered-CAD images (from different camera perspectives) per CAD code sample. For objects requiring multiple extrudes (e.g., object 3), \ours{}, given the real images, struggles to capture correct CAD operations for these more “complex” solids. Improved generation for these more complicated CAD models might be obtained by incorporating real images of geometric solids into Stage 1 training, so the MLP layers learn a better “description” of these types of images for the LLM.

\subsubsection{Performance on Unseen CAD Operations}
We suspected that our approach of fine-tuning a foundation VLM on CAD code would enable the model to extend to CAD operations it did not explicitly see during fine-tuning. To test this theory, we planned an experiment to ask \ours{} to add fillets to generated solid models. Filleting is a good test of \ours{}'s extensibility, as none of the examples in our fine-tuning \dataset{} included CadQuery code with fillet commands, but CadQuery supports filleting operations (using the \texttt{solid.edges().fillet(fillet\_radius)} command). When prompted (in a second query) to add fillets to all edges of a simple rectangular prism CAD model it had previously generated given an input image, \ours{} could not do it. Even when the filleting prompt was explicitly provided -- reminding the model of the correct syntax of the CadQuery filleting command (see Figure \ref{fig:fillet}) -- \ours{} would output an extended version of its original code that made no use of the specified fillet command and was syntactically incorrect.

\begin{figure*}[t]
\begin{center}
\setlength{\unitlength}{0.012500in}%
\includegraphics[width=0.8\linewidth]{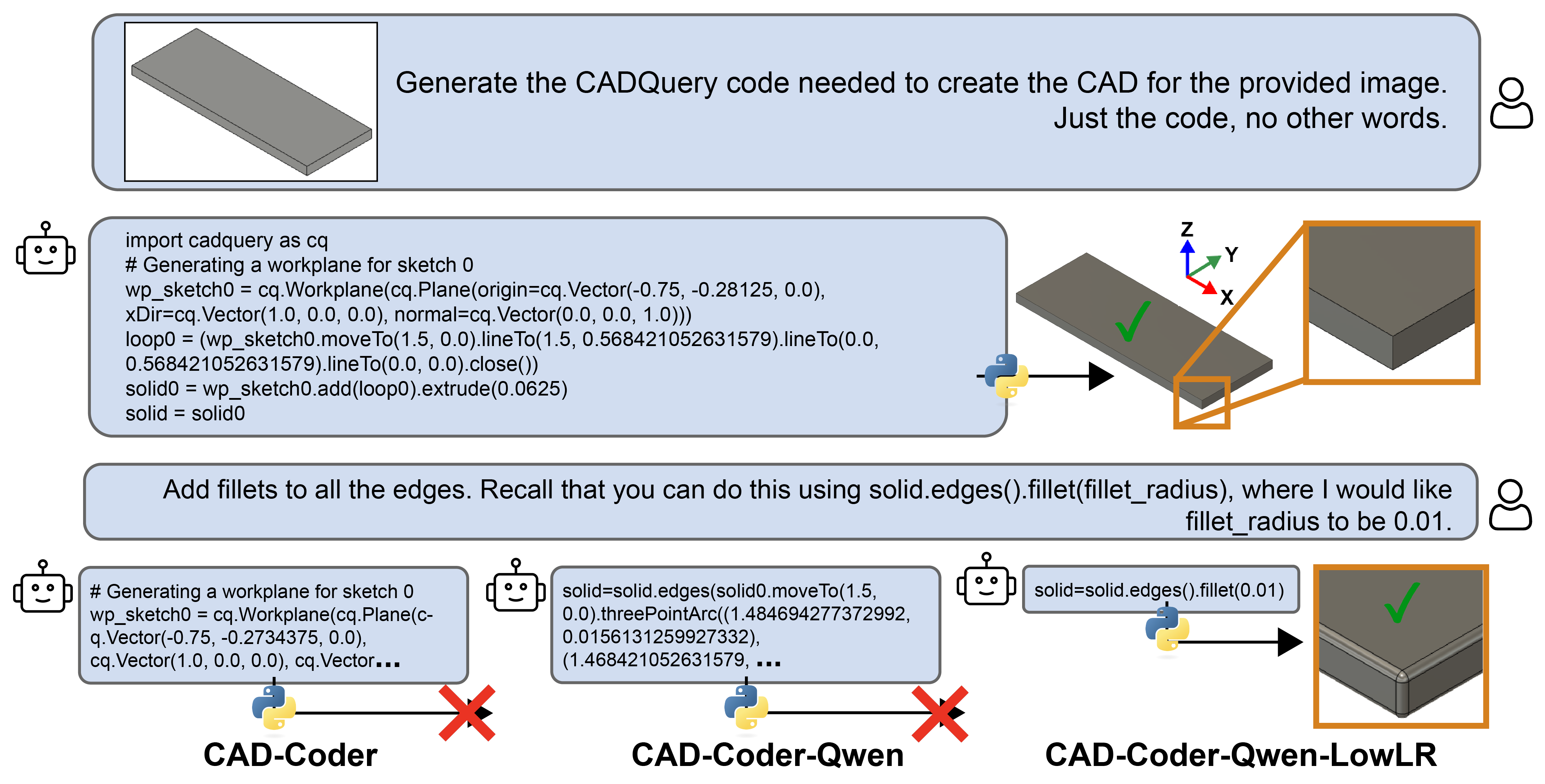}
\end{center}
\caption{Examples of \ours{} variants attempting to add fillets to CAD solids. The figure compares the performance of \ours{}, \ours{}-Qwen2.5-14B, and \ours{}-Qwen2.5-14B-LowLR given identical filleting prompts. Only the LowLR variant correctly applies the fillet operations.}

\label{fig:fillet} 
\end{figure*}

We suspected that \ours{}'s inability to fillet may stem from the pre-trained LLM's (Vicuna-13B-v1.5) lack of knowledge about CadQuery. To better characterize Vicuna-13B-v1.5's knowledge of CadQuery, we prompted the text-only model to generate CadQuery code of a 4x5x6 unit box with fillets on all edges. The model was not able to generate syntactically correct code. To remedy this issue of \ours{}'s limited knowledge of CadQuery, we experimented with training \ours{} using a different pre-trained LLM. Qwen2.5-Coder-32B-Instruct is the best performing model on an LLM coding leaderboard (Hugginface's bigcode-models-leaderboard~\footnote{https://huggingface.co/spaces/bigcode/bigcode-models-leaderboard, as of March 17, 2025}), and we tested a smaller variant (14B instead of 32B, due to compute constraints) for evaluation on its ability to generate CadQuery code. When given the same filleted-box text-only prompt that Vicuna-13B-v1.5 struggled with, Qwen2.5-Coder-14B-Instruct output a CadQuery script that produced the desired solid model.

We trained another model, \ours{}-Qwen2.5-14B, following the same architecture and training procedure as \ours{} except Qwen2.5-Coder-14B-Instruct was used as the pre-trained LLM in favor of Vicuna-13B-v1.5. Stage 2 training for \ours{}-Qwen2.5-14B took 3.31 hours on 8 H100 GPUs. We suspected that this model -- given its pre-existing knowledge of CadQuery -- would perform significantly better than \ours{} at adding fillets to a CAD model's edges. However, when given the same two-query filleting question as \ours{}, \ours{}-Qwen2.5-14B produced syntactically invalid code, incorrectly using the \texttt{.edges()} CadQuery method (see Figure \ref{fig:fillet}). This result was surprising, since we knew that Qwen2.5-Coder-14B-Instruct had accurate knowledge of these methods before its fine-tuning in the \ours{} pipeline. This result indicates that \ours{}-Qwen2.5-14B may have lost some of its pre-trained CadQuery knowledge during fine-tuning.

In an attempt to help preserve Qwen2.5-Coder-14B-Instruct's pre-trained knowledge, we trained one more variant of \ours{}, which we call \ours{}-Qwen2.5-14B-LowLR, identical to \ours{}-Qwen2.5-14B except with a halved learning rate during Stage 2 training (1e-5 instead of 2e-5). This model also took 3.31 hours to train on 8 H100 GPUs. When given the same two-query filleting questions as the other two \ours{} variants, \ours{}-Qwen2.5-14B-LowLR successfully adds fillets to the generated solid (Figure \ref{fig:fillet}). This result suggests that with tuning of the Stage 2 training hyperparameters, \ours{}-type models have the potential to generalize to CAD operations not explicitly included in the fine-tuning dataset. It is important to note that \ours{}-Qwen2.5-14B-LowLR's extendability to unseen CAD operations is not robust. The filleting prompt used (Figure \ref{fig:fillet}) is quite prescriptive; more abstract prompts (e.g. ``add fillets to all edges'') that do not remind the model about the \texttt{.fillet()} method usage do not elicit the same correct behavior from \ours{}-Qwen2.5-14B-LowLR. 
Future work will investigate training strategies that more effectively preserve pre-trained LLM knowledge in VLM fine-tuning.

\begin{table}[th]
\centering
\caption{Comparison of variants of \ours{} with different pre-trained LLMs and learning rates, evaluated by Valid Syntax Rate (VSR) and Intersection-Over-Union (IOU\textsubscript{best}) scores on the 100-sample test subset.}
\label{tab:qwen-variants}
\resizebox{\columnwidth}{!}{%
\begin{tabular}{@{}lcccc@{}}
\toprule
\textbf{Model} & \textbf{Pre-trained LLM} & \makecell{\textbf{Learning} \\ \textbf{Rate}} & \textbf{VSR $\uparrow$} & \textbf{IOU\textsubscript{best} $\uparrow$} \\ 
\midrule
\ours{} & Vicuna-13B-v1.5 & 2e-5 & \textbf{100\%} & \textbf{0.675} \\
\midrule
\makecell[l]{\ours{}- \\ Qwen2.5-14B} & \makecell{Qwen2.5-Coder-\\14B-Instruct} & 2e-5 & 95\% & 0.641 \\
\midrule
\makecell[l]{\ours{}- \\ Qwen2.5-14B-LowLR} & \makecell{Qwen2.5-Coder-\\14B-Instruct} & 1e-5 & 94\% & 0.592 \\
\addlinespace
\bottomrule
\end{tabular}%
}%
\end{table}

In Table \ref{tab:qwen-variants}, we also quantify \ours{}-Qwen2.5-14B and \ours{}-Qwen2.5-14B-LowLR according to the VSR and IOU\textsubscript{best} metrics used in Section \ref{sec:compare_with_baselines}. Even though it leverages a code-specific pre-trained LLM, \ours{}-Qwen2.5-14B unexpectedly performs marginally worse than \ours{} on both the VSR and IOU\textsubscript{best} metrics. The 5/100 syntactically incorrect CadQuery scripts that \ours{}-Qwen2.5-14B generates all correspond with relatively complex CAD models, but whose ground-truth scripts fall within the 4096 token limit. In generating these scripts, \ours{}-Qwen2.5-14B exceeds the maximum token limit during generation, so the scripts are truncated mid-command. Despite Stage 2 training on the \dataset{} (limited to 4096 token examples), \ours{}-Qwen2.5-14B's long generation behavior may be a result of the fact that its LLM, Qwen2.5-Coder-14B-Instruct, was pre-trained to handle much longer context lengths (up to 131k tokens). In contrast, \ours{}'s pre-trained LLM, Vicuna-13B-v1.5, was pre-trained to handle context lengths up to 4096 tokens. While further work is needed to better adapt longer context pre-trained LLMs (>4096 tokens) to our shorter context length dataset, Qwen2.5-Coder-14B-Instruct's longer context limit will be useful in adapting this work to more complex, longer token-length CAD scripts.

In comparison to \ours{}-Qwen2.5-14B, \ours{}-Qwen2.5-14B-LowLR performs worse in terms of IOU\textsubscript{best}. This result illustrates that lowering the learning rate -- to preserve pre-training knowledge of CadQuery and improve generalization on unseen CAD operations -- comes at the cost of reduced performance on the fine-tuning task. The trade-off between improvement on the fine-tuning task and the retention of pre-trained knowledge is a known problem for foundation models~\cite{luo2023empirical}. Several fine-tuning strategies~\cite{mukhoti2023fine, heyrani2025activation} have explored alternatives to end-to-end fine-tuning to better balance the two. We intend to explore these in future work.

\section{Limitations, Future Work and Conclusion}
While \ours{} demonstrates significant advancements in CAD code generation, it also exhibits some limitations. Firstly, despite its improved generalization to real-world images, the model occasionally struggles to accurately infer precise dimensions and proportions from images taken from varying perspectives. This indicates that further work on perspective and real-world image robustness, possibly through multi-view training or image augmentation strategies, is required. Additionally, although our experiments showed some capability to generalize to unseen CAD operations, such generalization remains heavily prompt-dependent, highlighting the need for enhanced strategies to preserve and leverage pre-existing LLM knowledge during fine-tuning.

Future research will also focus on improving the robustness and generalizability of \ours{} by expanding the training dataset and exploring more model architectures. Efforts will be directed towards incorporating multi-view and perspective-invariant image inputs into training to improve model accuracy on real-world images. Developing an extensive test-set of real-world images paired with CAD code would also help to rigorously quantify the performance of \ours{} and other models on the real-photo-to-CAD-code task. Other promising directions include exploring continual learning to better retain and extend pre-trained VLM knowledge of CAD operations. Utilizing reasoning-based LLM models or incorporating chain-of-thought logic into the fine-tuning dataset could also improve the accuracy of generated 3D solids.

This paper introduced \ours{}, a Vision-Language Model explicitly fine-tuned for generating accurate and editable CadQuery code directly from visual inputs. Leveraging the new \dataset{}, we demonstrated substantial improvements over existing image-conditioned, code-generating baselines, achieving state-of-the-art syntactic validity and geometric accuracy. Our findings illustrate the significant potential of fine-tuned VLMs to streamline and enhance engineering design workflows with editable CAD code generation, setting a strong foundation for future research into automated CAD modeling.

\bibliographystyle{asmems4}

\section*{Acknowledgments}
The authors gratefully acknowledge MIT-IBM for their partial support of this work. This material is based upon work supported by the National Science Foundation Graduate Research Fellowship. Any opinion, findings, and conclusions or recommendations expressed in this material are those of the authors(s) and do not necessarily reflect the views of the National Science Foundation.

\bibliography{asme2e}

\appendix
\section{Optimal Solid Model Alignment}
\label{app:iou}
In this section, we detail the continuous solid shape (Rigid Body) alignment problem formally and demonstrate how two shapes are aligned to measure the intersection over union (IOU) for two given shapes. Firstly, we treat solid objects as rigid bodies in 3-dimensional space, that is to say, an arbitrary solid in space is represented as an in-definite set $\Omega \subset \mathbb{R}^3$ with boundary $\Gamma:\partial{\Omega}$. Given such a solid body we seek to identify the optimal transformation to align any two arbitrary solid geometries.

First, we will analyze the case of two identical~(in shape) rigid bodies that have been transformed in space arbitrarily, that is to say, that they have been rotated, scaled, and translated arbitrarily. Under this assumption Lemma \ref{lm:bijection} implies there exists a relative volume preserving bijection, $f: \Omega_1 \rightarrow \Omega_2$ between the two sets. Note that given the indefinite~(continuous) nature of the bodies, hence making the cardinality of sets the same, there exist infinitely many bijections between the two sets. Here we assume that we know the ideal bijection between the identical shapes. Ultimately, the assumption made is that we know the solution to our problem exists, and we formulate the normalization that will allow us to actually find this solution. More intuitively, $f$ maps the points in $\Omega_1$ to the exact same part of the shape in $\Omega_2$ with respect to the shape. This assumption comes with no loss of generality in our shape analysis, and we will not assume anything about $f$ in general other than it's existence and relative volume preservation. The existence of $f$, is implied by the fact that the shapes are assumed to be identical, meaning they are identical manifolds transformed with an \textbf{invertible affine} transformation~(See Lemma \ref{lm:bijection}). Ultimately $f$ represents the correspondence between the points in the two sets in 3D Euclidean space, similar to the correspondence often seen in point-based shape matching, made continuous in an abstract sense~(no assumptions other than existence are needed for our purposes). Given this bijection, we define the \textit{distance}/\textit{difference} between two solid bodies as:

\begin{equation}
    d(\Omega_1,\Omega_2):=\int_{\Omega_1}\|\mathbf{x}-f(\mathbf{x})\|_2^2 d\mathbf{x}
\end{equation}

Where $d\mathbf{x}$ is the volume differential in Euclidean space. This is essentially the volume integral of the distance between the points in $\Omega_1$ and $\Omega_2$ given bijection $f$. Now we formulate our problem of aligning said shapes, namely finding the transformation~(inverse of $f$) such that if applied to $\Omega_1$ maps it onto $\Omega_2$, as:

\begin{equation}
\begin{split}
    \min_{\mathbf{R},s,\mathbf{t}} \quad & d = \int_{\Omega_1}\|s\mathbf{R}\mathbf{x}+\mathbf{t}-f(\mathbf{x})\|_2^2 d\mathbf{x}\\
    \text{s.t.} \quad& \mathbf{R} \in \mathbf{SO}(3) \\
    & \mathbf{R}\mathbf{R^T}=\mathbf{I} \\
    & \mathbf{t}\in\mathbb{R}^3\\
    & s>0
\end{split}
\end{equation}

Where $\mathbf{R}$ is a rotation matrix, $s$ is a real number representing the scale, and $\mathbf{t}$ is a translation vector in this setting. In Lemma \ref{lm:aligment} we attempt to solve this alignment problem for identical shapes. We show that under this assumption of $f$'s existence, we can solve the problem given two solids. However, we do not necessarily have identical shapes when comparing a target CAD model with a candidate reconstruction. However, we can establish a shape normalization scheme that guarantees identical shapes will be in perfect coincidence. Noting Lemma \ref{lm:aligment}, we propose the operator:

\begin{equation}
    \label{eqn:normalize}
    n(\Omega)=\{\frac{\mathbf{x}-\bar{\mathbf{x}}}{\sqrt{\frac{\operatorname{tr}(I)} {2\times \mathrm{V o l} ( \Omega_{2} )}}}\mid\mathbf{x}\in \Omega\},
\end{equation}

where $\bf I$ is the matrix of inertia for solid $\Omega$. Finally, we observed in Lemma \ref{lm:aligment}, that the ideal alignment will align the principle axes of two solids, as such given two solids we first normalize scale and translation using (\ref{eqn:normalize}) and then align the principle axes of inertial from $\bf I$ for each solid. Note that the alignment of the principal axes is ambiguous as the direction of the principal axes is not known. As such there are 8 possible ways to align said angles of which 4 are valid rotations in $SO(3)$, meaning that in practice we perform all 4 alignments and pick the one with the best IOU.

\begin{lemma}
\label{lm:bijection}
Let \(\Omega_1 \subset \mathbb{R}^3\) be a bounded solid with nonzero volume and let
\[
\Omega_2 = \{\, s\,\mathbf{R}\mathbf{x} + \mathbf{t} \mid \mathbf{x} \in \Omega_1 \,\}
\]
be its image under an affine transformation where \(\mathbf{R} \in SO(3)\) (so that \(\det(\mathbf{R}) = 1\)), \(s > 0\), and \(\mathbf{t} \in \mathbb{R}^3\). Then the mapping
\[
f : \Omega_1 \to \Omega_2,\quad f(\mathbf{x}) = s\,\mathbf{R}\mathbf{x} + \mathbf{t},
\]
is a bijection satisfying the following relative volume preservation property: For any integrable function \(g:\Omega_2\to\mathbb{R}\),
\[
\int_{\Omega_2} g(\mathbf{y})\, d\mathbf{y} = \frac{\operatorname{Vol}(\Omega_2)}{\operatorname{Vol}(\Omega_1)} \int_{\Omega_1} g\bigl(f(\mathbf{x})\bigr)\, d\mathbf{x}.
\]
\end{lemma}

\begin{proof}
\textbf{1. Bijectivity:}  
The transformation
\[
f(\mathbf{x}) = s\,\mathbf{R}\mathbf{x} + \mathbf{t}
\]
is an affine map. Since \(\mathbf{R}\) is a rotation (and thus invertible) and \(s>0\), the matrix \(s\,\mathbf{R}\) is invertible. Consequently, \(f\) is one-to-one and onto, with the inverse given by
\[
f^{-1}(\mathbf{y}) = \mathbf{R}^T \left(\frac{\mathbf{y}-\mathbf{t}}{s}\right).
\]
\textbf{2. Jacobian Determinant and Volume Scaling:}  
Since \(f\) is affine, its differential \(Df(\mathbf{x})\) is constant and equal to \(s\,\mathbf{R}\) for all \(\mathbf{x}\in\Omega_1\). Its Jacobian determinant is
\[
\det\bigl(s\,\mathbf{R}\bigr) = s^3\,\det(\mathbf{R}) = s^3,
\]
because \(\det(\mathbf{R})=1\) given $R\in SO(3)$.

\textbf{3. Change of Variables and Relative Volume Preservation:}  
By the change-of-variable, for any integrable function \(g\) defined on \(\Omega_2\),
\[
\int_{\Omega_2} g(\mathbf{y})\, d\mathbf{y} = \int_{\Omega_1} g\bigl(f(\mathbf{x})\bigr) \left|\det\bigl(Df(\mathbf{x})\bigr)\right|\, d\mathbf{x} = s^3 \int_{\Omega_1} g\bigl(f(\mathbf{x})\bigr)\, d\mathbf{x}.
\]
In particular, choosing \(g(\mathbf{y})\equiv1\) yields
\[
\operatorname{Vol}(\Omega_2) = s^3\,\operatorname{Vol}(\Omega_1) \Rightarrow s^3= \frac{\operatorname{Vol}(\Omega_2)}{\operatorname{Vol}(\Omega_1)},
\]

where $\operatorname{Vol}(\Omega_1)=\int_{\Omega_1} d\mathbf{x}$ and $\operatorname{Vol}(\Omega_2)=\int_{\Omega_2} d\mathbf{x}$ This demonstrates that the mapping \(f\) not only provides a bijective correspondence between points in \(\Omega_1\) and \(\Omega_2\), but it also preserves the relative volume distribution.
\end{proof}

\begin{lemma}[Optimal Rigid‐Body Alignment]
\label{lm:aligment}
Let $\Omega_1$ and $\Omega_2$ be two identical in shape rigid bodies in $\mathbb{R}^3$, assumed to be related by a transformation:

\[
f : \Omega_1 \to \Omega_2,\quad f(\mathbf{x}) = s^*\,\mathbf{R^*}\mathbf{x} + \mathbf{t^*}
\]

Define
\[
\min_{\mathbf{R}\in SO(3),\, s>0,\, \mathbf{t}\in \mathbb{R}^3}
\int_{\Omega_1}\bigl\|\,s\,\mathbf{R}\,\mathbf{x} + \mathbf{t} - f(\mathbf{x})\bigr\|_2^2\,d\mathbf{x}.
\]
Then there exists a unique solution:

\[
\mathbf{R}^* = \bf VU^T,
\]
\[
s^{*}=\frac{\sqrt{\frac{\operatorname{tr}(\mathbf{I_2})} {\mathrm{V o l} ( \Omega_{2} )}}} {\sqrt{\frac{\operatorname{tr}(\mathbf{I_1})} {\mathrm{V o l} ( \Omega_{1} )} }},
\]
\[
t^* = \bar{x}_2 - sR\bar{x}_1,
\]

where $\bf I_1$ and $\bf I_2$ are the matrices of inertia for solid $\Omega_1$ and $\Omega_2$ respectively and:
\[
\bar{x}_1 = \frac{\int_{\Omega_1} \bf x\,d\mathbf{\mathbf{x}}}{\int_{\Omega_1} \,\mathbf{dx}},\quad \bar{x}_2 = \frac{\int_{\Omega_2} \bf x\,\mathbf{dx}}{\int_{\Omega_2} \,\mathbf{dx}},
\]
\[
{\bf S}=\int_{\Omega_{1}} (f ( {\bf x} )-\bar{\mathbf{x}}_2) \, (\bf x- \bar{\mathbf{x}}_1)^{T} \, d {\bf x} = \bf U\Sigma V^T,
\]
that achieves this minimum.
\end{lemma}
\begin{proof}
First, we note that by Lemma \ref{lm:bijection}, $f$ is a volume preserving bijection. We will use this fact in our derivation of the solution. To find the ideal translation vector we differentiate the objective with respect to the translation and find the root of the gradient:

\begin{equation}
    \frac{\partial d}{\partial t} = \frac{\partial}{\partial t} \int_{\Omega_1}\|s\mathbf{R}\mathbf{x}+\mathbf{t}-f(\mathbf{x})\|_2^2 d\mathbf{x} = 2 \int_{\Omega_1} \left( s\,R\,x + t - f(x) \right) \, d\mathbf{x}.
\end{equation}

Note that we can move the differentiation operation inside the integral given the Leibniz rule. Given this, we can find the root of the gradient to obtain the optimal translation:

\begin{equation}
\begin{split}
    \int_{\Omega_1} \left( s\,\mathbf{R}\,x + t^* - f(x) \right) \, d\mathbf{x} = \\s\,\mathbf{R}\int_{\Omega_1} x\,d\mathbf{x} + \mathbf{t^*}\int_{\Omega_1} \,d\mathbf{x} - \int_{\Omega_1} f(x)\,d\mathbf{x} = \\ s\,\mathbf{R}\,\frac{\int_{\Omega_1} x\,d\mathbf{x}}{\int_{\Omega_1} \,d\mathbf{x}} + \mathbf{t^*} - \frac{\int_{\Omega_1} f(x)\,d\mathbf{x}}{\int_{\Omega_1} \,d\mathbf{x}} = 0 \\ \Rightarrow \mathbf{t^*}= \frac{\int_{\Omega_1} f(x)\,d\mathbf{x}}{\int_{\Omega_1} \,d\mathbf{x}} - s\,\mathbf{R}\,\frac{\int_{\Omega_1} x\,d\mathbf{x}}{\int_{\Omega_1} \,d\mathbf{x}}.
\end{split}
\end{equation}

Now let $\bar{x}_1$, and $\bar{x}_2$ be the centroids of $\Omega_1$ and $\Omega_2$ respectively:

\begin{equation}
    \bar{x}_1 = \frac{\int_{\Omega_1} x\,\mathbf{d\mathbf{x}}}{\int_{\Omega_1} \,d\mathbf{x}},\quad \bar{x}_2 = \frac{\int_{\Omega_2} x\,d\mathbf{x}}{\int_{\Omega_2} \,d\mathbf{x}}
\end{equation}

by lemma \ref{lm:bijection} we have:

\begin{equation}
    \mathbf{t^*} = \bar{\mathbf{x}}_2 - sR\bar{\mathbf{x}}_1
\end{equation}

we can see that the sets $\Omega_1$ and $\Omega_2$ can be normalized with respect to translation by removing the centroids, defining $\tilde{\Omega_1}$ and $\tilde{\Omega_2}$ as:

\begin{align}
    \tilde{\Omega}_1 = \{x - \bar{x}_1\mid x\in\Omega_1\} \\
    \tilde{\Omega}_2 = \{x - \bar{x}_2\mid x\in\Omega_2\}
\end{align}

With this normalization, we have the variable change:
\begin{equation}
    \tilde{\mathbf{x}}=\mathbf{x}-\bar{\mathbf{x}_1}\quad \text { and } \quad \tilde{f}(\mathbf{x})=f(\mathbf{x})-\bar{\mathbf{x}_2},
\end{equation}

which allows us to rewrite our objective as:
\begin{equation}
    d(s, \mathbf{R})=\int_{\Omega_1}\|s \mathbf{R} \tilde{\mathbf{x}}-\tilde{f}(\mathbf{x})\|_2^2 d \mathbf{x}.
\end{equation}

From Lemma \ref{lm:bijection}, we assumed that there exists $f(\mathbf{x})=s^*\bf R^*x + t^*$ which means $\tilde{f}(\mathbf{x})=s^*\bf R^*x$ given our solution for $\bf t^*$. Here we can show that $s^*$ can be deduced given the change of variable with translation. Lemma \ref{lm:bijection} implies:

\begin{equation}
    \int_{\Omega_2} \|\mathbf{x}-\bar{\mathbf{x}}_2\|^2 = \frac{\operatorname{Vol}(\Omega_2)}{\operatorname{Vol}(\Omega_2)}\int_{\Omega_1} \|s^*\mathbf{R^*(x-\bar{\mathbf{x}}_1)}\|^2 
\end{equation}

Note $\|\mathbf{R}\mathbf{x}\|=\|\mathbf{x}\|$ given $\mathbf{R}\in SO(3)$, therefore:

\begin{equation}
    \int_{\Omega_2} \|\mathbf{x}-\bar{\mathbf{x}}_2\|^2 = {s^*}^2 \frac{\operatorname{Vol}(\Omega_2)}{\operatorname{Vol}(\Omega_1)}\int_{\Omega_1} \|\mathbf{x}-\bar{\mathbf{x}}_1\|^2 
\end{equation}

This implies:

\begin{equation}
   s^{*}=\frac{\sqrt{\frac{\int_{\Omega_{2}} \| \mathbf{x}-\bar{\mathbf{x}}_{2} \|^{2} \, d \mathbf{x}} {\mathrm{V o l} ( \Omega_{2} )}}} {\sqrt{\frac{\int_{\Omega_{1}} \| \mathbf{x}-\bar{\mathbf{x}}_{1} \|^{2} \, d \mathbf{x}} {\mathrm{V o l} ( \Omega_{1} )} }}.
\end{equation}

Let $\bf I_1$ and $\bf I_2$ be the tensors of inertia for $\Omega_1$ and $\Omega_2$, by definition $\int_{\Omega_{1}} \| \mathbf{x}-\bar{\mathbf{x}}_{1} \|^{2} \, d \mathbf{x}= \frac{1}{2}\operatorname{tr}(I_1)$ and $\int_{\Omega_{2}} \| \mathbf{x}-\bar{\mathbf{x}}_{2} \|^{2} \, d \mathbf{x}= \frac{1}{2}\operatorname{tr}(I_2)$, therefore:

\begin{equation}
   s^{*}=\frac{\sqrt{\frac{\operatorname{tr}(I_2)} {\mathrm{V o l} ( \Omega_{2} )}}} {\sqrt{\frac{\operatorname{tr}(I_1)} {\mathrm{V o l} ( \Omega_{1} )} }}.
\end{equation}

Note that this is under the Euclidean norm objective, and the L1 norm would yield a different scaling (Rotation would not preserve distance under L1 norm in Euclidean space). Now we can introduce the following symbols to simplify expressions:

\begin{align}
    \mathbf{A}=&\int_{\Omega_1} \tilde{\mathbf{x}} \tilde{\mathbf{x}}^T d \mathbf{x}, \quad \text{Note:} \int_{\Omega_1}\|\tilde{\mathbf{x}}\|^2 d \mathbf{x}=\mathrm{t r}(\mathbf{A}_1) = \frac{1}{2}\operatorname{tr}(\mathbf{I}_1)\\
    {\bf S}=&\int_{\Omega_{1}} \tilde{f} ( {\bf x} ) \, \tilde{\bf x}^{T} \, d {\bf x}, \quad \text{Note:} \int_{\Omega_{1}} \langle\mathbf{R} \, \tilde{\mathbf{x}}, \, \tilde{f} ( \mathbf{x} ) \rangle\, d \mathbf{x}=\mathrm{t r} \Big( \mathbf{R} \, \mathbf{S} \Big) 
\end{align}

Noting that $\int_{\Omega_1}\|\tilde{f}(\mathbf{x})\|^2 d \mathbf{x}$ is a constant and we will refer to it as $C$ below:

\begin{equation}
\label{eqn:shiftedobj}
d ( s, \mathbf{R} )=s^{2} \mathrm{~ t r} ( \mathbf{A} )-2 s \mathrm{~ t r} ( \mathbf{R} \, \mathbf{S} )+ C
\end{equation}

We can see that rotation only affects the objective only in the term $-2 s \mathrm{~ t r} ( \mathbf{R} \, \mathbf{S} )$ and since $s>0$ the optimal rotation $\mathbf{R}^*$ is equivalent to the solution to $\max_{\mathbf{R}^*\in SO(3)}\mathrm{~ t r} ( \mathbf{R} \, \mathbf{S} )$, which is an orthogonal Procrustes problem which can be solved by singular value decomposition~\cite{procrustesbook}. Let $\mathbf{S}=\mathbf{U} \boldsymbol{\Sigma} \mathbf{V}^{T} $ be the singular value decomposition of $\bf S$, then $\mathrm{~ t r} ( \mathbf{R} \, \mathbf{S} )=\mathrm{~ t r} ( \bf V^TRU \Sigma)$ which given all of $\bf U,V,R$ are orthogonal will be maximized when $\bf V^TRU=I$, hence:

\begin{equation}
    \bf R^*=VU^T
\end{equation}

Let:

\begin{align}
    \mathbf{A}_1=&\int_{\Omega_1} \tilde{\mathbf{x}} \tilde{\mathbf{x}}^T d \mathbf{x}\\
    \mathbf{A}_2=&\int_{\Omega_1} \tilde{f}(\mathbf{x}) \tilde{f}(\mathbf{x})^T d \mathbf{x}, \quad
\end{align}

Note that by definition of the matrix of inertia~(namely $\mathbf{I}=\int_{\Omega} \left( \| \mathbf{x}-\bar{\mathbf{x}} \|^{2} \mathbf{I}_{3}-( \mathbf{x}-\bar{\mathbf{x}} ) ( \mathbf{x}-\bar{\mathbf{x}} )^{T} \right) d \mathbf{x}.$), we have:

\begin{align}
\label{eqn:Ai}
    \mathbf{A}_{i}=\frac{1}{2}\mathrm{t r} ( \mathbf{I}_{i} ) \, \mathbf{I}_{3}-\mathbf{I}_{i} \quad \forall i\in\{1,2\}
\end{align}

Moreover, given the assumption of $f$'s existence, plugging in the optimal transformation we have:

\begin{equation}
\label{eqn:A2}
\begin{split}
    \mathbf{A}_2=&\int_{\Omega_1} \tilde{f}(\mathbf{x}) \tilde{f}(\mathbf{x})^T d \mathbf{x}\\
    =&\int_{\Omega_1} s^*\mathbf{R^*}\tilde{\mathbf{x}}\tilde{\mathbf{x}}^T\mathbf{R^*}^T s^* d \mathbf{x}\\
    =&{s^*}^2\mathbf{R^*}A_1\mathbf{R^*}^T
\end{split}
\end{equation}

Noting the fact that the eigenvectors of $\bf A_i$ are the same as $\bf I_i$ from (\ref{eqn:Ai}) we can see that interestingly this rotation aligns the principal axes of inertia for the two solids without ambiguity. However, in certain settings, one may choose to perform all possible 4 valid alignments of principle axes~(8 possibilities with directional ambiguity of eigenvectors of the matrix of inertia, 4 of which will be valid rotations in $SO(3)$) instead of finding the continuous covariance matrix $S$ and performing a singular value decomposition on it. Hence, the unique minimizing transformation parameters $\bf R^*$, $s^*$, $\bf t^*$  are given by the formulas above, completing the proof.
\end{proof}

\end{document}